\documentclass[11pt]{article}
\usepackage{fullpage}
\usepackage{natbib}
\usepackage{authblk}
\usepackage{graphicx} 
\usepackage{bm}
\usepackage[utf8]{inputenc}
\usepackage[T1]{fontenc}
\usepackage{graphicx,amsmath,amsfonts,amssymb,bm,url,color,latexsym,amsthm}
\usepackage{natbib}
\usepackage[hidelinks]{hyperref}
\usepackage{wrapfig}
\usepackage[small,bf]{caption}
\usepackage[outdir=./]{epstopdf}
\usepackage{subcaption}
\usepackage{bbm}
\usepackage{microtype}
\usepackage[toc, page]{appendix}
\usepackage[ruled]{algorithm2e}
\usepackage{algpseudocode}
\usepackage{xr-hyper}
\usepackage[capitalize]{cleveref}
\usepackage{verbatim}
\usepackage{framed}
\usepackage{comment}
\usepackage{tikz}
\usepackage{fge}
\usepackage{mathdots}
\usepackage{enumitem}

\newtheorem{theorem}{Theorem}
\newtheorem{lemma}{Lemma}

\newtheorem{definition}{Definition}
\newtheorem{assumption}{Assumption}

\newtheorem{remark}{Remark}

\newcommand{\mb}{\mathbf}

\newcommand{\bb}{\mathbb}

\newcommand{\eps}{\varepsilon}

\newcommand{\E}{\bb E}

\newcommand{\norm}[2]{\left\| #1 \right\|_{#2}}
\newcommand{\abs}[1]{\left| #1 \right|}

\renewcommand{\P}{\mathbb{P}}

\DeclareMathOperator{\diag}{diag}

\DeclareMathOperator{\rank}{rank}

\DeclareMathOperator{\KL}{KL}

\DeclareMathOperator*{\argmin}{argmin}

\DeclareFontFamily{U}{mathx}{\hyphenchar\font45}
\DeclareFontShape{U}{mathx}{m}{n}{
      <5> <6> <7> <8> <9> <10>
      <10.95> <12> <14.4> <17.28> <20.74> <24.88>
      mathx10
      }{}
\DeclareSymbolFont{mathx}{U}{mathx}{m}{n}
\DeclareFontSubstitution{U}{mathx}{m}{n}
\DeclareMathAccent{\widecheck}{0}{mathx}{"71}

\newcommand{\bs}{\boldsymbol}

\author[1]{Dian Jin\thanks{dj370@scarletmail.rutgers.edu}}
\author[1]{Yuqian Zhang\thanks{yqz.zhang@rutgers.edu}}
\author[2]{Qiaosheng Zhang\thanks{zhangqiaosheng@pjlab.org.cn}}

\affil[1]{Department of Electrical and Computer Engineering, Rutgers University, New Brunswick}
\affil[2]{Shanghai Artificial Intelligence Laboratory}
\begin{document}
\title{Community Detection for Contextual-LSBM: Theoretical Limitations
of Misclassification Rate and Efficient Algorithms}
\maketitle
\begin{abstract}
   The integration of network information and node attribute information has recently gained significant attention in the community detection literature. In this work, we consider community detection in the Contextual Labeled Stochastic Block Model (CLSBM), where the network follows an LSBM and node attributes follow a Gaussian Mixture Model (GMM). Our primary focus is the misclassification rate, which measures the expected number of nodes misclassified by community detection algorithms. We first establish a lower bound on the optimal misclassification rate that holds for any algorithm. When we specialize our setting to the LSBM (which preserves only network information) or the GMM (which preserves only node attribute information), our lower bound recovers prior results. Moreover, we present an efficient spectral-based algorithm tailored for the CLSBM and derive an upper bound on its misclassification rate. Although the algorithm does not attain the lower bound, it serves as a reliable starting point for designing more accurate community detection algorithms (as many algorithms use spectral method as an initial step, followed by refinement procedures to enhance accuracy).
\end{abstract}

\section{Introduction}

The community detection problem \cite{girvan2002community,newman2007mixture,fortunato2010community,karrer2011stochastic} is one of the fundamental challenges in the field of network science and data analysis. It focuses on identifying and reconstructing the underlying community structure within a network or graph. Accurately identifying these communities is crucial for understanding the structure and function of complex systems, ranging from social networks and biological systems to communication networks and financial markets.


Among the various proposed models, the Stochastic Block Model (SBM) \cite{holland1983stochastic} stands out as one of the most significant and extensively studied in recent decades. Given parameters $n$ and $K$, according to the definition of the SBM, each node with index $i \in [n]$ belonging to the set of nodes $\mathcal{V}$ is independently assigned to the $k$-th community (where $k \in [K]$) with probability $\alpha_k$. Denote $\sigma:[n]\rightarrow [K]$ as the ground truth indicator function such that $\sigma(i)=k$ if the node $i$ belongs to community $k$, thus the community detection problem is to \textbf{infer} the ground truth indicator function $\sigma(\cdot)$ from the observed graph (usually represented by an adjacency matrix $\mb A$). Here the probability of observing an edge $\mb A_{ij}$ between any two nodes $i, j \in [n]$ is characterized by $\mathbb{P}_{\sigma(i), \sigma(j)}$.
In addition to the basic SBM, many variants have been proposed and investigated, such as weighted/labeled SBM \cite{jog2015information,aicher2015learning,nowicki2001estimation,yun2016optimal}, degree-corrected SBM \cite{karrer2011stochastic,gao2018community,yan2014model,qin2013regularized}, overlapping/mixed-membership SBM \cite{airoldi2008mixed,latouche2014model,fu2009dynamic,xing2010state}, hypergraph SBM \cite{ghoshdastidar2014consistency,Zhang2023exact,cole2020exact,Nandy2024degree}, latent space model \cite{hoff2002latent,rohe2011spectral,bing2020adaptive,Rohe2023varimax,bing2023optimal}. Among the aforementioned variants, in this paper we focus on the \emph{labeled SBM (LSBM)} \cite{heimlicher2012community,yun2016optimal}, an extension of the SBM in terms of the non-binary observed edges. Compared to the basic SBM, the non-binary edge formulation provides enhanced capability in modeling diverse types of interactions between any pair of nodes.

In the community detection literature, many existing methods rely solely on topological information, such as adjacency or Laplacian matrices. However, real-world applications often require the analysis of contextual networks, as side information like node attributes can play a crucial role in uncovering node labels within communities. Significant research has been dedicated to this area, as evidenced by works such as \cite{zhang2016community, abbe2022lpPCA, braun2022iterative, yan2021covariate}. Specifically, \cite{zhao2021optimal,jin2021exact,abbe2022lpPCA, dreveton2024exact, braun2022iterative} investigate the threshold for \emph{exact recovery} (equavilent to  \emph{strong consistency} \cite{abbe2018community}) of the Contextual-SBM, with \cite{zhao2021optimal,jin2021exact,abbe2022lpPCA} focusing on the special case of two communities and \cite{dreveton2024exact, braun2022iterative} extending the results to the more general case of \( K \) communities. While exact recovery is a valid performance metric that has received significant theoretical attention, it requires the number of misclassified nodes to be \emph{exactly zero}, which is overly stringent and often impractical. In contrast, most real-world applications of community detection or node classification often prioritize \emph{weak consistency}, where a small fraction of nodes is allowed to be misclassified. This  aligns more closely with practical needs, where partial misclassification is acceptable and more realistic.

Motivated by practical considerations, in this paper, we investigate a model that integrates the LSBM with contextual Gaussian node attributes, referred to as \emph{Contextual-LSBM (CLSBM)}. Our focus is on the relationship between the \emph{misclassification rate} and the model parameters. Specifically, let \( s \) represent the expected number of misclassified nodes. Our primary objective is to establish the lower bound of the misclassification rate for the CLSBM when \( s = o(n) \). While this result has been established for the LSBM in \cite[Theorem 1]{yun2016optimal}, our \cref{thm:main} is the first to extend this result to the setting where contextual node attributes are present. Notably, \cite{braun2022iterative} and \cite{dreveton2024exact} also investigate contextual models similar to ours.  However, neither \cite{braun2022iterative} nor \cite{dreveton2024exact} allows edges to carry labels, and \cite{dreveton2024exact} focuses specifically on achieving exact recovery. These can be viewed as special cases or reductions of our more general results (see \cref{thm:main} and \cref{remark_D}). In addition to the lower bound, we also design an efficient community detection algorithm for CLSBM. The algorithm begins by aggregating topological and contextual information into a latent factor model and then employs a spectral-based method to consistently estimate the community assignments for each node. 
\subsection{Organization}
The remainder of the paper is organized as follows: In \cref{sec:main_contri}, we review related literature and highlight our main contributions. In \cref{sec:model_def}, we formalize the model definitions and discuss the assumptions underpinning for our theoretical findings. \cref{sec:lowerbound} provides the theoretical lower bound for the number of misclassified nodes in the model and explores its connections with earlier work. Furthermore, \cref{subsec:spec_meth} introduces an efficient spectral community detection algorithm for the proposed aggregated factor model and presents theoretical guarantees for its performance. Finally we conclude our results and future directions in \cref{sec:concl}.
\subsection{Notations}
In this paper, bold lowercase letters, such as $\mb x$, represent vectors, while bold uppercase letters, like $\mb X$, denote matrices and $3$-dimensional arrays. For a matrix $\mb X$, $\mb X_{ij}$ indicates the element located at the $i$-th row and $j$-th column, with $\mb X_{i\cdot}$ and $\mb X_{\cdot j}$ representing the $i$-th row and $j$-th column of $\mb X$, respectively. For a $3$-dimentional array $\mb X$, we use $\mb X(i,j,k)$ to indicate the element at indices $i, j, k$ for each dimension. Occasionally, $\mb x_j$ is used as a simplified notation for $\mb X_{\cdot j}$. 

The operators of probability and expectation are denoted by $\mathbb P$ and $\mathbb E$, respectively. For any integer $K \in \mathbb N$, the set $\{1, 2, 3, \cdots, K\}$ is represented by $[K]$. The notation $\norm{\mb x}{q}$, $\norm{\mb X}{F}$, $\norm{\mb X}{op}$ represent the $\ell_p$ norm of a vector $\mb x$, Frobenius norm and operator norm of matrix $\mb X$ respectively. 
Here we use $\mb 1(\cdot)$ to display the indicator function. For two sequences $a_n$ and $b_n$, the notation $a_n \lesssim b_n$ means $a_n \leq C b_n$ for some constant $C$.


\section{Related work and main contributions}\label{sec:main_contri}
The threshold for exact recovery in the weighted/labeled SBM was first investigated by~\cite{jog2015information}, which extends the results established in \cite{zhang2016minimax} for the SBM. Moreover, \cite{yun2016optimal} derived the optimal  misclassification rate for the LSBM and introduced a two-stage algorithm to achieve the information-theoretic limit. Besides, the Contextual-SBM, a generalization of SBM that incorporates additive node attributes to provide further information for identifying communities, has been explored in various studies \cite{braun2022iterative,deshpande2018contextual,dreveton2024exact,abbe2022lpPCA}. However, to the best of our knowledge, the optimal misclassification rate for models with general non-binary labeled edges combined with independent arbitrary node attributes still remains unexplored. Our \cref{thm:main} establishes that the lower bound for the expected number of misclassified nodes (as defined in \cref{def:misrate}) is $n\exp\left(-(1+\epsilon_n) D(\bs\alpha,\mb P, \bs\mu) n\right)$ for some sequence $\epsilon_n$ approaching zero. Here the key quantity $D(\bs\alpha, \mb P, \bs \mu)$ (defined in \eqref{def:D}) is the divergence defined based on the minimal value for the sum of two KL divergences, representing the topological and attribute information respectively. Furthermore, \cref{remark_D} reveals a connection between the minimal value of sum of KL divergences over constrained sets and the Chernoff divergence, extending the results of \cite[Claim $4$]{yun2016optimal} to our settings and consequently prove that the aformentioned lower bound indeed match the optimal misclassification rate in the LSBM (\cite{zhang2016minimax, yun2016optimal}) and the Gaussian mixture model (\cite{lu2016statistical}). Our main technical contribution stems in the analysis to the problem \eqref{def:D}. It is worth noting that, while the optimal solution for \(D_A\) has been investigated in \cite{yun2016optimal}, the impact of the newly introduced term \(D_X\) on the main results remains unclear. The additional term introduces further challenges, as it involves in solving an optimization problem in function spaces, which is significantly more complex. In addition to establishing the theoretical lower bound, we also apply spectral-based methods to the proposed aggregated latent factor model, achieving a polynomial misclassification rate. Despite the suboptimality of the result, the output of the spectral-based method provides a valuable initialization for further refinement in a subsequent step (e.g., \cite{gao2018community, yun2016optimal, braun2022iterative}).

\section{Model description and main result}

\subsection{Model definition}\label{sec:model_def}
We largely follow the notation introduced in \cite{yun2016optimal} to describe the LSBM. Denote $\mathcal{V}$ as the set of nodes constrained to $\textit{card}(\mathcal{V})=n$ and let $\{\mathcal{V}_k\}_{k\in[K]}$ represent the $K$ disjoint communities. Here each node is independently assigned to a community \(\mathcal{V}_k\) with probability \(\alpha_k\), where \(k \in [K]\) and $\mb 1^T\mb \alpha=1$, and this assignment is independent of the number of nodes \(n\). Each edge \((i, j) \in \mathcal{V}^2\) is assigned to a label \(l \in \{0,1,\cdots, L\}\), where the probability of labeling depends on \(i\), \(j\), and \(l\), and is denoted by \(\mb{P}(\sigma(i), \sigma(j), l)\). Here, \(\mb{P} \in \mathbb{R}^{K \times K \times (L+1)}\), and \(\sigma(\cdot)\) is the ground truth indicator function. We summarize the above in the definition of LSBM as follows:

\begin{definition}[LSBM]
    Let  $L\geq 1$, $K\geq 2$, $\bs \alpha\in \bb R_{+}^{K}$, $\mb P\in \mathbb R_{+}^{K\times K\times(L+1)}$ and
    we denote $\left\{\mb A_l\right\}_{l\in \{0,1,\cdots, L\}}\sim \textit{LSBM} \left(L, n, K, \mb P,\bs \alpha,\sigma(\cdot)\right)$ if $\forall l\in \{0,1,\cdots, L\}$ $\mb A_l\in \left\{0, 1\right\}^{n\times n}$ is symmertical and its upper triangular elements are independent, such that:
    \begin{align}
       &\P(\mb A_l(i,j)=1) =\mb P(\sigma(i),\sigma(j),l),~\sum_{l=0}^{L}\mb A_l(i,j)=1,\nonumber\\
       &\P(\sigma(i)=k)=\alpha_k,
    \end{align}
    $\forall i\neq j\in [n],\,k\in [K]$.
    Here we assume $\mb A_0(i,i)=1$ for all $i\in [K]$ to eliminate the self-loops within the graph.
    
\end{definition}
In addition to topological information, contextual node attributes are also incorporated into the final model. We assume that each node $i$ is associated with an attribute vector $\mathbf{x}_i \in \mathbb{R}^d$, and the graph $\left\{\mb A_l\right\}_{l\in \{0,1,\cdots, L\}}$ is considered independent of the sets of attribute vectors $\{\mathbf{x}_i\}_{i\in [n]}$ given the ground truth indicator function $\sigma(\cdot)$. This setting was initially proposed in \cite{zhang2016community} and has since been discussed in subsequent works \cite{binkiewicz2017covariate,yan2021covariate,weng2016community,deshpande2018contextual}.

\begin{definition}[CLSBM]
Let $d\geq 1$, $L\geq 1$, $K\geq 2$, $\bs \alpha\in \bb R_{+}^K$, $\mb P\in \mathbb R_{+}^{K\times K\times(L+1)}$, $\bs \mu\in \mathbb R^{d\times K}$, and we denote $(\left\{\mb A_l\right\}_{l\in \{0,1,\cdots, L\}}, \left\{\mb x_i\right\}_{i\in [n]})\sim \textit{CLSBM}(L, d, n, K, \mb P,\bs \alpha,\sigma(\cdot))$, if $\{\mb A_l\}_{l\in \{0,1,\cdots, L\}}\sim \textit{LSBM} \left(L, n, K, \mb P,\bs \alpha,\sigma(\cdot)\right)$, and $\left\{\mb x_i\right\}_{i\in [n]}\sim \mathcal{N}(\boldsymbol{\mu}_{\sigma(i)}, \mb I_d)$.
    
\end{definition}
Let $\bar{\sigma}(\cdot)$ be an assignment vector function that indicates the assignment for each node and is estimated from the observed edge labels $\{\mb A_l\}_{l\in \{0,1,\cdots, L\}}$ and node attributes $\left\{\mb x_i\right\}_{i\in [n]}$.
In the next section, we derive an information-theoretic lower bound for the expected number of misclassified nodes, valid for all the possible choices of $\tilde{\sigma}(\cdot)$, defined as follows
\begin{align}\label{def:misrate}
\bar{s}:=\inf_{\substack{\pi\in S_K\\
\tilde{\sigma}\in \mathcal{F}_{I}}}~\E\left[\sum_{i\in [n]} \mb 1_{\{\pi(\tilde{\sigma}(i))\neq \sigma(i)\}}\right].
\end{align}
 Here, \( S_K \) represents the set of all permutations of \([K]\), and \(\mathcal{F}_{I}\) denotes the set of all indicator functions. Before presenting the main results, we first outline our assumption as follows.
\begin{assumption}\label{ass:1}
$\exists~ \eta_1,\eta_2>0$ such that 
    \begin{align}
        &\forall l\in\{0,1,\cdots, L\},~\forall i,j,k\in [K],\quad \frac{\mb P(i,j,l)}{\mb P(i,k,l)}\leq \eta_1,\quad\textit{and}\quad\norm{\bs \mu_{i}}2\leq \eta_2.
    \end{align}
\end{assumption}
Although counterintuitive, the aforementioned assumption implies that the signal contributed by any one of the 
$K$ communities cannot be excessively strong for both the network and node attributes. This assumption arises from the limitations of our proof techniques, and it could potentially be eliminated with improved proof mechanisms. Specifically, the assumption for \(\bs{\mu}\) is newly introduced here to bound the second-order moments of the pseudo-likelihood estimator for the perturbed model in the proof of \cref{thm:main}. Similarly, the assumption regarding \(\mb{P}\) originates from \cite{yun2016optimal}, serves the the same purpose. It is also worth mentioning that similar assumptions on \(\bs{\mu}\) are commonly found in previous literature, such as \cite{abbe2022lpPCA}.

\subsection{Lower bound on the number of misclassification nodes}\label{sec:lowerbound}
In this section we provide one of our main results. We start with the definition of $D(\bs \alpha,\mb P,\boldsymbol{\mu})$:
\begin{align}\label{def:D}
    &D(\bs \alpha,\mb P,\boldsymbol{\mu}):=\min_{k_1\neq k_2\in [K]}\min_{\substack{
    \mb q_{A}\in \mathbb R_{+}^{K\times L}\\
    q_{X}\in \mathcal{F}
    }} D_{A}(k_1, k_2,\mb P, \mb q_{A})+\frac{1}{n}D_{X}(k_1, k_2,\bs \mu, q_{X}),
    \end{align}
where \begin{align}\label{def:Dx_Da}
    &D_{A}(k_1, k_2,\mb P, \mb q_{A}):=\max\biggl\{\sum_{k=1}^K\alpha_k \KL_{A}\left(\mb q_{A}(k)~\Vert~ \mb P(k, k_1)\right),\sum_{k=1}^K\alpha_k\KL_{A}\left(\mb q_{A}(k)~\Vert~ \mb P(k, k_2)\right)\biggr\},\nonumber\\
    & D_{X}(k_1, k_2,\bs \mu, q_{X}):=\max\biggl\{ \KL_{X}\left(q_{X}~\Vert~ p_{ X}(k_1)\right),\KL_{X}\left(q_{X}~\Vert~p_{ X}(k_2)\right)\biggr\}.
\end{align}
Here $\KL_{A}$ and $\KL_{X}$ represent the KL divergence for discrete and continuous probability respectively, and $\mathcal{F}$ represents the set of probability density function (PDF) such that $\mathcal{F}:=\left\{f:\mathbb R^{d}\rightarrow \mathbb R,\,\textit{s.t}\, \int_{x} f(x) dx=1 \right\}$,  and $p_{ X}(k)$ is denoted as the PDF of $\sim \mathcal{N}(\bs \mu_{k},\mb I_d)$.
The following theorem then builds the lower bound of the number of misclassified nodes in expectation with the aforementioned divergence $D(\bs \alpha,\mb P,\boldsymbol{\mu})$:
\begin{theorem}[Lower bound]\label{thm:main}
Denote $\bar{p}:=\max_{i,j,l\geq 1}\mb P(i,j,l)$, grant \cref{ass:1} and assume $\bar{p}=\omega(1/n)$, $\bar{p}=o(1)$, and $\eta_2=o(n)$. Let $s=o(n)$. If there exists an algorithm that asymptotically has fewer misclassified nodes than $s$ in expectation, i.e., $\limsup_{n\rightarrow \infty}\frac{\bar{s}}{s}\leq 1$, then we have 
    \begin{align}
        \liminf_{n\rightarrow \infty} \frac{n D(\bs \alpha, \mb P, \boldsymbol{\mu})}{\log(n/s)}\geq 1,
    \end{align}
    where  $\bar{s}$ is defined in \eqref{def:misrate}.
\end{theorem}
\begin{proof}
    We defer the proof to \cref{sec:proof_thm_main}.
\end{proof}
\cref{thm:main} implies that the expected number of misclassified nodes for the proposed CLSBM is at least of the order $n \exp(-n D(\bs \alpha, \mb P, \boldsymbol{\mu}))$. 
The following remark indicates that the proposed measure $D(\bs \alpha, \mb P, \boldsymbol{\mu})$ generalizes both the CH-divergence for the LSBM \cite[Claim 4]{yun2016optimal}  and the minimal distance between the centers of Gaussian mixtures.
\begin{remark}\label{remark_D}
    When $\bar{p}=o(1)$, for CLSBM, one has
    \begin{align}
        &D(\bs \alpha, \mb P, \boldsymbol{\mu})\nonumber\\
        &\!\!=\min_{k_1\neq k_2\in [K]}\max_{t\in [0,1]}\sum_{k=1}^K\alpha_k\biggl[\sum_{l=1}^{L}\biggl((1-t)\mb P(k, k_1,l)+t \mb P(k, k_2,l)-\mb P^{1-t}(k, k_1,l)\mb P^{t}(k, k_2,l)\biggr)\!\!+\frac{t(1-t)}{2n}\norm{\bs \mu_{k_1}-\bs\mu_{k_2}}2^2\biggr].
    \end{align}
The proof is deferred \cref{proof_remark_D}. When the contextual node attributes are absent (i.e., \(\bs{\mu} = 0\)), the divergence \(D(\bs{\alpha}, \mb{P}, \bs{\mu})\) simplifies to the CH-divergence introduced in \cite{yun2016optimal}. On the other hand, in cases where topological information is missing (e.g., \(\mb{P}(k, k_1, l) = \mb{P}(k, k_2, l)\) for all \(k, k_1, k_2, l\)), we have \(D(\bs{\alpha}, \mb{P}, \bs{\mu}) = \frac{1}{8n} \|\bs{\mu}_{k_1} - \bs{\mu}_{k_2}\|_2^2\), which corresponds to threshold for the Gaussian mixture model, as shown in \cite[Theorem 3.3]{lu2016statistical}.
\end{remark}

\subsection{Spectral community detection for latent factor model}\label{subsec:spec_meth}
In this section, we extend the proposed CLSBM to a latent factor model and demonstrate that the spectral-based method can consistently estimate the proposed estimator \( s(\hat{\sigma}) \) defined in \eqref{def:spec_esti}, based on the adjacency matrices \( \{\mb A_{l}\}_{l\in \{0,1,2,\dots, L\}} \) and contextual node attribute vectors \( \{\mb x_i\}_{i\in [n]} \). Specifically, we first outline the construction of the estimator by combining the topological information with the contextual node attributes and provide an intuitive explanation of the spectral-based method in Section \ref{subsubsec:aggr}. Subsequently, in the following section, we formalize the spectral-based method in Section \ref{subsubsec:alg_theo} and establish its consistency, as proven in \cref{lem:spec}, thereby demonstrating its effectiveness.



\subsubsection{Aggregated latent factor model}\label{subsubsec:aggr}
The latent factor model demonstrates its capability in revealing hidden low-dimensional structures within observed high-dimensional data and has been widely applied in various domains, including PCA, Gaussian mixtures, nonnegative matrix factorization, sparse factor analysis and so on. This provides an intuitive basis for combining both the topological model and contextual node attributes into a shared latent space, as both the adjacency matrix and the node attributes share the same eigenspace spanned by the ground truth assignment matrix $\mb Z(\sigma) \in \mathbb{R}^{n \times k}$, where $\mb Z_{i\cdot} := \mb e_{\sigma(i)}^T$. To elaborate, we begin by defining our aggregator factor matrix $\mb S$: Given weights $\{w_l\}_{l \in [L]}$ i.i.d sampled from a uniform distribution over $[0,1]$, we define:
\begin{align}\label{def:S}
    \mb S := \sum_{l=1}^{L} w_l \mb A_{l} + \frac{1}{n} \mb X^T \mb X,
\end{align}
where $\mb X \in \mathbb{R}^{d \times n}$ is a matrix whose columns represent the collection of $n$ observed node attributes. The weights $\{w_l\}_{l \in [L]}$ are employed to identify signals for different labels $l \in [L]$, and the factor $1/n$ balances the signal-to-noise ratio (SNR) between the graph signal and the Gaussian signal, as motivated by \cref{remark_D}. Let $\mb P_s := 1/(2L) \sum_{l=1}^L \mb P(\cdot,\cdot,l) \in \mathbb{R}^{K \times K}$. At the population level, the following holds:
\begin{align}\label{eqn:exp_s}
    \mathbb{E}\left[\mb S\right] &= \mathcal{H}(\mb Z \mb P_s \mb Z^T) + \frac{1}{n} \mb Z \bs \mu^T \bs \mu \mb Z^T \nonumber \\
    &= \mb Z \mb P_s \mb Z^T - \diag(\mb Z \mb P_s \mb Z^T) + \frac{1}{n} \mb Z \bs \mu^T \bs \mu \mb Z^T \nonumber \\
    &= \underbrace{\mb Z \left(\mb P_s + \frac{1}{n} \bs \mu^T \bs \mu \right) \mb Z^T}_{\mb M} - \diag(\mb Z \mb P_s \mb Z^T).
\end{align}
Here, $\mathcal{H}(\cdot)$ denotes the hollowing operation, which sets the diagonal elements of the matrix to zero, and $\mb Z$ is shorthand for $\mb Z(\sigma)$. It is evident that the first term $\mb M$ in $\mathbb{E}\left[\mb S\right]$ is a symmetric low-rank factor model when $K<n$, which has been extensively studied in areas such as low-rank matrix decomposition, factor analysis, signal processing, and spectral clustering. By definition, the optimal solution of the $k$-means algorithm applied to the columns/rows of $\mb M$ should yeild the ground truth assignment $\mb Z$. In practice, however, instead of directly applying $k$-means to the columns of $\mb M$ (which has dimension $n$), spectral methods \cite{newman2006modularity} are often pre-employed due to their computational effficiency. Specifically, let $\mb U_K\in \mathbb R^{n\times K}$ denote the top-$K$ eigenvectors of $\mb M$. We apply $k$-means to the columns of $\mb U_K^T \mb M$. In fact applying $k$-means to the columns of $\mb M$ or $\mb U_K^T \mb M$ yeilds equivalent clustering outcomes, as $\mb U_K$ is an orthonormal matrix that preserves distances during projection. Chooseing $\mb U_K^T \mb M$ is motivated by its computational efficiency for clustering, particularly because $K$ is typically much smaller than $n$. However, since neither $\mb M$ nor $\mb U_K$ is observed even at the population level, we substitute $\mb S$ for $\mb M$ and employ its top-$K$ eigenvectors, denoted as $\bar{\mb U}_K$, in place of $\mb U_K$. Given that the second term in \eqref{eqn:exp_s} has a small operator norm, i.e., $\|\text{diag}(\mb Z \mb P_s \mb Z^T)\|_{\text{op}} = o(1)$, the misclassification rate of the $k$-means algorithm based on $\mb S$ depends on $\|\mb S - \mathbb{E}(\mb S)\|_F$, a quantity can be further bounded via some concentration inequalities. We summarize the aforementioned procedures in \cref{alg:sp}, refer to as Spectral Community Detection.


    


\begin{algorithm}
\caption{Spectral Community Detection}
\label{alg:sp}
\begin{algorithmic}[1]
\Require Node attributes $\left\{\mb x_i\right\}_{i\in [n]}$, adjacency matrices $\left\{\mb A_l\right\}_{l\in \{0,1,\cdots, L\}}$, and number of communities $K$.
\Ensure  Assignment for each node

\State \textbf{Step 1: Compute eigenvectors}
\Statex Calculate the eigen decomposition of $\mb S$ as proposed in \eqref{def:S}:
\begin{align}
    \mb S=\sum_{i=1}^{n}\sigma_i\bar{\mb u_i}\bar{\mb u_i}^T.\nonumber
\end{align}
and denote $\bar{\mb U}_K\in \mathbb{R}^{n\times k}$ as the collection of the top $K$ eigenvectors.

\State \textbf{Step 2: Apply $k$-means}
\Statex Denote $\mb Q:=\bar{\mb U}_K^T\mb S\in \mathbb R^{k\times n}$ and denote $\bar{\sigma}(\cdot)$ as the solution to $k$-means such that:
\begin{align}
    \bar{\sigma}(\cdot), \left\{\bar{\bs\theta}\right\}_{i\in [K]}=\argmin_{\substack{\sigma\in \mathcal{F}_a,\\\bs\theta_i\in \mathbb R^k} } \sum_{i\in [n]}\norm{\mb Q_{\cdot i}-\bs \theta_{\sigma(i)}}2^2.\nonumber
\end{align}

\State \textbf{Return:} Assignment indicator function $\bar{\sigma}(\cdot)$.
\end{algorithmic}
\end{algorithm}

\vspace{-.1in}
\subsubsection{Theoretical guarantee}\label{subsubsec:alg_theo}

Denote $s(\bar{\sigma})$ as the number of misclassified nodes for indicator function $\bar{\sigma}$ recovered from \cref{alg:sp} such that
\begin{align}\label{def:spec_esti}
    s(\bar{\sigma}):=\inf_{\pi\in S_K}~\E\left[\sum_{i\in n} \mb 1_{\{\pi(\bar{\sigma}(i))\neq \sigma(i))\}}\right].
\end{align}
The following lemma thus provides the upper bound of $s(\bar{\sigma})$.
\begin{lemma}\label{lem:spec}
    Grant the assumptions in \cref{thm:main} and further assume $\bar{p}=\omega (\log(n)/n)$, let $\bar{\sigma}$ be the solution of the  \cref{alg:sp}, we have
    \begin{align}
        s(\bar{\sigma})\lesssim \frac{K}{n\cdot \mathrm{SNR}}
    \end{align}
    with high probability, where  $\mathrm{SNR}:=\min_{k_1\neq k_2}\norm{\left(\mb P_s\right)_{k_1}-\left(\mb P_s\right)_{k_2}+\frac{1}{n}\left(\bs \mu_{k_1}-\bs \mu_{k_2}\right)^T\bs \mu}2^2$. Furthermore if $n\cdot \mathrm{SNR}=w(1)$ we obtain $s(\bar{\sigma})=o(1)$. 
\end{lemma}
    \vspace{-.1in}
\begin{proof}
    The proof is deferred to \cref{sec:proof_lem_spec}.
\end{proof}
\vspace{-.1in}
Lemma \ref{lem:spec} demonstrates that the proposed indicator \( s(\bar{\sigma}) \) is upper bounded by the inverse of the defined (SNR) in a polynomial manner, which corresponds to the minimum distance between any two columns in the matrix \( \mb P_s + \frac{1}{n} \bs \mu^T \bs \mu \). Rates of this nature have been extensively studied in the literature \cite{lei2015consistency, loffler2021optimality, zhang2024leave} for spectral-based methods under SBM or Gaussian mixtures, respectively. However, our work is the first to provide theoretical guarantees for spectral-based methods under the CLSBM. Despite achieving consistency, our analysis does not lead to the optimal rate presented in \cref{thm:main}, which is exponentially small. However, we believe that this estimate can serve as an effective initializer for potential subsequent refinement steps, as demonstrated in \cite{yun2016optimal, gao2018community}. By leveraging the spectral method's output as a starting point, further refinement techniques such as MLE can be applied to improve the accuracy of community detection, potentially bridging the gap between polynomial and exponential misclassification rates.
\vspace{-.05in}
\section{Conclusion}\label{sec:concl}
In this paper, we derive the theoretical lower bound for the misclassification rate in the Label-Stochastic Block Model with Gaussian node attributes. Our results also establish connections with previously introduced models, highlighting their broader applicability and usefulness. Beyond the theoretical lower bound, we propose a spectral method with theoretical guarantees for the proposed model, achieving a polynomial error rate in estimating community cluster assignments with high probability. In future work, we aim to explore the theoretical limitations of the spectral method to determine whether it can attain the optimal exponential rate.
 {
\bibliographystyle{abbrvnat}

}
\appendix
\section{Proof of \cref{thm:main}}\label{sec:proof_thm_main}
Mimicking the change of measure method in \cite{yun2016optimal} we adapt it to the CLSBM model and its perturbed version in this section. Denote $\Phi$ and $\Psi$ as the parameter for the true parameter model (CLSBM) and a potential perturbed parameter model respectively. To be more specific, the true ground truth model $\Phi$ first generates the cluster for each node based on $\bs \alpha$ such that $\P(\sigma(i)=k)=\alpha_k$ for the $i$-th node indepedently. For the topological information,  model $\Phi$ generates the observed labels $l\in [L]$ for two different nodes $i$, $j$ based on the $\mb P(\sigma(i),\sigma(j), l)$. For the contextual information, $\mb x_i$ is sampled from $\mathcal{N}(\bs \mu_{\sigma(i)}, \mb I_d)$.  We state the definition of $\Psi$ as follows:
\paragraph{Definition of the perturbed model $\Psi$}:
Denote $(k_1^\star, k_2^\star, \mb q_{A}^\star, q_{X}^\star)$ as the solve for $D(\bs\alpha, \mb P, \bs \mu)$ such that 
\begin{align}
    D(\bs\alpha, \mb P, \bs \mu)=D_{A}(k_1^\star, k_2^\star, \mb q_{A}^\star)+\frac{1}{n}D_{X}(k_1^\star, k_2^\star, q_{X}^\star)
\end{align}
The edge label generation mechanism for the topological information($\mb A$) in $\Psi$ is the same as in \cite{yun2016optimal} with replace $\mb q$ with $\mb q_{\mb A}^\star$, Here, we describe the generation of labels for the side information $\{\mb x_i\}_{i\in [n]}$ as follows: let $i^\star=\argmin_{\sigma(i)\in \{k_1^\star, k_2^\star\}} i$. 
\begin{itemize}
    \item When $i \neq i^\star$: $\mb x_i$ is sampled from $\mathcal N(\bs \mu_{\sigma(i)}, \mb I_d)$ which is the same as in $\Phi$.
    \item When $i = i^\star$: $\mb x_{i^\star}$ is sampled from the distribution whose pdf is $q_{\mb X}^\star$.
\end{itemize}

We denote $\log{\frac{\mathrm{d}\mathbb{P}_{\Psi}}{\mathrm{d}\mathbb P_{\Phi} }}$ as the pseudolikelihood ratio of the observed labels and contextual variables such that:   
\begin{align}\label{def:log-ratio}
\log{\frac{\mathrm{d}\mathbb{P}_{\Psi}}{\mathrm{d}\mathbb P_{\Phi} }}:=\sum_{i\neq i^{\star}}^{n} \log\frac{\mb q_{A}^\star(\sigma(i),\mb A_{i,i^{\star}})}{\mb P(\sigma(i), \sigma(i^\star),A_{i,i^{\star}})}+\log\frac{q_{X}^\star}{p_{ X}(\sigma(i^\star))}.
\end{align}


In the following two sections, we provide the detailed proof of theorem \ref{thm:main}, which  follows the proof in \cite{yun2016optimal}. The proof consists of two parts, the first part constructs a stochastic model $\Psi$ and lower bounds the misclassification error  $\E_{\Phi}(\epsilon^{\pi}(n))$ using the expectation of the proposed log-likelihood ratio $\log{\frac{\mathrm{d}\mathbb{P}_{\Psi}}{\mathrm{d}\mathbb P_{\Phi} }}$ defined in \eqref{def:log-ratio}; In the second part, we further estimate upper bounds for $\E_{\Psi}(\log{\frac{\mathrm{d}\mathbb{P}_{\Psi}}{\mathrm{d}\mathbb P_{\Phi} }})$ and  $\E_{\Psi}\left[\left(\log{\frac{\mathrm{d}\mathbb{P}_{\Psi}}{\mathrm{d}\mathbb P_{\Phi} }}-\E_{\Psi}[\log{\frac{\mathrm{d}\mathbb{P}_{\Psi}}{\mathrm{d}\mathbb P_{\Phi} }}]\right)^2\right]$ respectively, to complete the proof.
\subsection{Part 1}
Leverage \cref{lem:D} we assert there exist $(k_1^\star, k_2^\star, \mb q_{A}^\star, q_{X}^\star)$ such that: 
\begin{align}\label{eqn:D}
   D(\bs \alpha,\mb P,\boldsymbol{\mu})&=\sum_{k=1}^K\alpha_k \KL_{A}\left(\mb q_{A}^\star(k), \mb P(k, k_1^\star)\right)+ \frac{1}{n}\KL_{X}\left(q_{X}^\star, p_{ X}(k_1^\star)\right)\nonumber\\
   &=\sum_{k=1}^K\alpha_k\KL_{A}\left(\mb q_{A}^\star(k), \mb P(k, k_2^\star)\right)+\frac{1}{n}\KL_{X}\left(q_{X}^\star, p_{ X}(k_2^\star)\right)
\end{align}
    Recall the definition of the log ratio \cref{def:log-ratio} we obtain:
    \begin{align}
\log{\frac{\mathrm{d}\mathbb{P}_{\Psi}}{\mathrm{d}\mathbb P_{\Phi} }}:=\sum_{i\neq i^{\star}}^{n} \log\frac{\mb q_{A}^\star(\sigma(i),\mb A_{i,i^{\star}})}{\mb P(\sigma(i), \sigma(i^\star),A_{i,i^{\star}})}+\log\frac{q_{X}^\star}{p_{ X}(\sigma(i^\star))}.
\end{align}
Mimicking the approach described in \cite{yun2016optimal}, we can derive inequality $(17)$ in \cite{yun2016optimal} adapted with our new  defined log-ratio $\log{\frac{\mathrm{d}\mathbb{P}_{\Psi}}{\mathrm{d}\mathbb P_{\Phi} }}$ as the necessary condition for $\E[\epsilon(n)] \leq s$:
\begin{align}\label{ieq:proof_prt1_proof_s}
    \log(n/s)-\log(2/\alpha_{k_1})\leq \E_{\Psi}(\log{\frac{\mathrm{d}\mathbb{P}_{\Psi}}{\mathrm{d}\mathbb P_{\Phi} }})+\sqrt{\frac{4}{\alpha_{k_1}}\E_{\Psi}\left[\left(\log{\frac{\mathrm{d}\mathbb{P}_{\Psi}}{\mathrm{d}\mathbb P_{\Phi} }}-\E_{\Psi}[\log{\frac{\mathrm{d}\mathbb{P}_{\Psi}}{\mathrm{d}\mathbb P_{\Phi} }}]\right)^2\right]}.
\end{align}
Then the next part finishes the proof by estimating the bounds of $\E_{\Psi}[\log{\frac{\mathrm{d}\mathbb{P}_{\Psi}}{\mathrm{d}\mathbb P_{\Phi} }}]$ and $\E_{\Psi}\left[\left(\log{\frac{\mathrm{d}\mathbb{P}_{\Psi}}{\mathrm{d}\mathbb P_{\Phi} }}-\E_{\Psi}\log{\frac{\mathrm{d}\mathbb{P}_{\Psi}}{\mathrm{d}\mathbb P_{\Phi} }}\right)^2\right]$ respectively.

\newpage
\subsection{Part 2}
\begin{itemize}
    \item Upper bound for $\E_{\Psi}[\log{\frac{\mathrm{d}\mathbb{P}_{\Psi}}{\mathrm{d}\mathbb P_{\Phi} }}]$:\\
    Denote $L:=\log{\frac{\mathrm{d}\mathbb{P}_{\Psi}}{\mathrm{d}\mathbb P_{\Phi} }}$ and we obtain 
\begin{align}\label{ieq:L_upper}
    &\E_{\Psi}[L]\nonumber\\
    &:=\E_{\Psi}\left[\sum_{i\neq i^{\star}}^{n} \log\frac{ q_A^\star(\sigma(i),A_{i,i^{\star}})}{\mb P(\sigma(i), \sigma(i^\star),A_{i,i^{\star}})}+\log\frac{q_{X}^\star}{p_{ X}(\sigma(i^\star))}\right].\nonumber\\
    &=\underbrace{\E_{\Psi}\sum_{i\neq i^{\star}}^{n} \log\frac{q_{A}^\star(\sigma(i),A_{i,i^{\star}})}{\mb P(\sigma(i), \sigma(i^\star),A_{i,i^{\star}})}}_{\Gamma_1}+\underbrace{\E_{\Psi}\log\frac{q_{X}^\star}{p_{ X}(\sigma(i^\star))}}_{\Gamma_2}
    \end{align}
     $\Gamma_1$ has been carefully studied in \cite[Theorem 1]{yun2016optimal} such that:
     \begin{align*}
         \Gamma_1\leq (n+2\log^2(n)\log(\eta_1))D_{A}(k_1^\star,k_2^\star, \mb q_A^\star).
     \end{align*} Thus we will focus on $\Gamma_2$: By definition of $\Psi$ we have
     \begin{align*}
         \Gamma_2= D_{X}(k_1^\star,k_2^\star, \mb q_X^\star).
     \end{align*}
Combine $\Gamma_1$ and $\Gamma_2$ together we have
\[
\E_{\Psi}[L]\leq (n+2\log^2(n)\log(\eta_1))D(\bs\alpha, \mb P, \bs\mu).
\]
Furthermore,
repeating the procedure in \eqref{ieq:L_upper} leads to  
\begin{align}
    \E_{\Psi}[L]\geq (n-2\log^2(n)\log(\eta_1))D(\bs \alpha, \mb p, \bs\mu).\nonumber
\end{align}
Consequently we obtain
\begin{align}
    \abs{\E_{\Psi}[L]-n D(\bs \alpha, \mb p, \bs\mu)}\leq 2\log^2(n)\log(\eta)D(\bs \alpha, \mb p, \bs\mu).
\end{align}

\item Upper bounding  $\E_{\Psi}\left[\left(\log{\frac{\mathrm{d}\mathbb{P}_{\Psi}}{\mathrm{d}\mathbb P_{\Phi} }}-\E_{\Psi}\log{\frac{\mathrm{d}\mathbb{P}_{\Psi}}{\mathrm{d}\mathbb P_{\Phi} }}\right)^2\right]$:
Recall the notation of $L$ we obtain:
\begin{align} \label{ieq:pt2_proof_L_variance}
&\E_{\Psi}\left[\left(\log{\frac{\mathrm{d}\mathbb{P}_{\Psi}}{\mathrm{d}\mathbb P_{\Phi} }}-\E_{\Psi}\log{\frac{\mathrm{d}\mathbb{P}_{\Psi}}{\mathrm{d}\mathbb P_{\Phi} }}\right)^2\right]\nonumber\\
&=\E_{\Psi}[\left(L-\E_{\Psi}[L]\right)^2]\nonumber\\
&=\E_{\Psi}[L^2]-\left(\E_{\Psi}[L]\right)^2\nonumber\\
&\leq \E_{\Psi}[L^2]-\left(n-2\log^2(n)\right)^2 D^2(\bs \alpha, \mb p, \mu).
\end{align}
Thus, the following part remains to upper bound $\E_{\Psi}[L^2]$: Without loss of generality we assume $\sigma(i^\star)=k_1$, by definition we have
\begin{align}\label{ieq:pt2_proof_exp_L_sq}
    &\E_{\Psi}[L^2]\nonumber\\
    &=\E_{\Psi}[L^2|~\sigma(i^\star)=k_1]\nonumber\\
    &=\underbrace{\P(i^\star\leq \log^2(n))\E_{\Psi}\biggl[L^2\Bigg|~\sigma(i^\star)=k_1, i^\star\leq \log^2(n)\biggr]}_{\Gamma_1}+\underbrace{\P(i^\star> \log^2(n))\E_{\Psi}\biggl[L^2\Bigg|~\sigma(i^\star)=k_1, i^\star> \log^2(n)\biggr]}_{\Gamma_2}
\end{align}
For $\Gamma_2$ we have \begin{align}
    \Gamma_2&=\P(i^\star> \log^2(n))\E_{\Psi}\biggl[L^2\Bigg|~\sigma(i^\star)=k_1, i^\star> \log^2(n)\biggr]\nonumber\\
    &\leq \frac{1}{n^4}\E_{\Psi}\left[\left(\sum_{i\neq i^{\star}}^{n} \log\frac{ q_A^\star(\sigma(i),A_{i,i^{\star}})}{\mb P(\sigma(i), \sigma(i^\star),A_{i,i^{\star}})}+\log\frac{q_{X}^\star}{p_{ X}(\sigma(i^\star))}\right)^2\Bigg|~\sigma(i^\star)=k_1, i^\star> \log^2(n)\right]\nonumber\\
    &=\frac{1}{n^4}\E_{\Psi}\biggl[\left(\sum_{i\neq i^{\star}}^{n} \log\frac{ q_A^\star(\sigma(i),A_{i,i^{\star}})}{\mb P(\sigma(i), \sigma(i^\star),A_{i,i^{\star}})}\right)^2+\left(\log\frac{q_{X}^\star}{p_{ X}(\sigma(i^\star))}\right)^2\nonumber\\
    &\quad\quad\quad+2\left(\sum_{i\neq i^{\star}}^{n} \log\frac{ q_A^\star(\sigma(i),A_{i,i^{\star}})}{\mb P(\sigma(i), \sigma(i^\star),A_{i,i^{\star}})}\right)\left(\log\frac{q_{X}^\star}{p_{ X}(\sigma(i^\star))}\right)\Bigg|~\sigma(i^\star)=k_1, i^\star> \log^2(n)\biggr]\nonumber\\
    &\leq \frac{1}{n^4}\left[n^2\log^2(\eta_1)+\E_{x\sim q_X^\star}\left[\norm{\mb x-\bs \mu_{k_1}}2^2-\norm{\mb x-(1-t)\bs \mu_{k_1}-t\bs\mu_{k_2}}2^2\right]^2+n\log(\eta_1)\KL_{X}(q_{X}^\star, p_X(\sigma(i^\star)))\right]\nonumber\\
    &\lesssim \frac{1}{n^4}\left(n^2\log^2(\eta_1)+t(1-t)\left(\eta_2^4+\eta_2^2\right)\right),\nonumber\\
    &\lesssim\frac{\log^2(\eta_1)}{n^2} + \frac{\eta_2^4 + \eta_2^2}{n^4}.
\end{align}
Here, we use $\P(i^\star > \log^2(n))\leq n^{-4}$ in the second step, the independence of $\{\mb x\}_{i\in [L]}$ and $\mb A$ in the third step, equation \eqref{eqn:q_X} and \cref{ass:1} in the fourth step and $0 < t < 1$ in the last step. 
To proceeds we then focus on $\Gamma_1$:
\begin{align}  \label{ieq:pt2_proof_gamma1}
\Gamma_1&\leq\E_{\Psi}\biggl[L^2\Bigg|~\sigma(i^\star)=k_1, i^\star\leq \log^2(n)\biggr]\nonumber\\
&= \E_{\Psi}\biggl[\left(\sum_{i\neq i^{\star}}^{n} \log\frac{q_{A}(\sigma(i),A_{i,i^{\star}})}{\mb P(\sigma(i), \sigma(i^\star),A_{i,i^{\star}})}+\log\frac{q_{X}^\star}{p_{ X}(\sigma(i^\star))}\right)^2\Bigg|~\sigma(i^\star)=k_1, i^\star\leq \log^2(n)\biggr]\nonumber\\
&=\E_{\Psi}\biggl[\left(\sum_{i> i^{\star}}^{n} \log\frac{q_{A}(\sigma(i),A_{i,i^{\star}})}{\mb P(\sigma(i), \sigma(i^\star),A_{i,i^{\star}})}+\log\frac{q_{X}^\star}{p_{ X}(\sigma(i^\star))}\right)^2\Bigg|~\sigma(i^\star)=k_1, i^\star\leq \log^2(n)\biggr]\nonumber\\
&\quad+\E_{\Psi}\biggl[\left(\sum_{i< i^{\star}} \log\frac{q_{A}(\sigma(i),A_{i,i^{\star}})}{\mb P(\sigma(i), \sigma(i^\star),A_{i,i^{\star}})}+\log\frac{q_{X}^\star}{p_{ X}(\sigma(i^\star))}\right)^2\Bigg|~\sigma(i^\star)=k_1, i^\star\leq \log^2(n)\biggr]\nonumber\\
&\quad+2\E_{\Psi}\biggl[\left(\sum_{i< i^{\star}} \log\frac{q_{A}(\sigma(i),A_{i,i^{\star}})}{\mb P(\sigma(i), \sigma(i^\star),A_{i,i^{\star}})}\right)\left(\sum_{i> i^{\star}}^{n} \log\frac{q_{A}(\sigma(i),A_{i,i^{\star}})}{\mb P(\sigma(i), \sigma(i^\star),A_{i,i^{\star}})}+\log\frac{q_{X}^\star}{p_{ X}(\sigma(i^\star))}\right)\nonumber\\
&\quad\quad\quad\Bigg|~\sigma(i^\star)=k_1, i^\star\leq \log^2(n)\biggr]\nonumber\\
&\leq \underbrace{\E_{\Psi}\biggl[\left(\sum_{i> i^{\star}}^{n} \log\frac{q_{A}(\sigma(i),A_{i,i^{\star}})}{\mb P(\sigma(i), \sigma(i^\star),A_{i,i^{\star}})}+\log\frac{q_{X}^\star}{p_{ X}(\sigma(i^\star))}\right)^2\Bigg|~\sigma(i^\star)=k_1, i^\star\leq \log^2(n)\biggr]}_{\Gamma_{11}}\nonumber\\
&\quad+\left(\log(n)\log(\eta_1)+\eta_2\right)^2+2\left(\log(n)\log(\eta_1)+\eta_2\right)D(\bs \alpha, \mb p, \mu).
\end{align}
Here we use the fact that any two nodes $i_1$ and $i_2$ satisfying $i_1 \neq i_2\geq i^\star$ are mutually independent in the derivation.  To upper bound $\Gamma_{11}$, we first have 
\begin{align}
     &\left(\sum_{i> i^{\star}}^{n} \log\frac{q_{A}(\sigma(i),A_{i,i^{\star}})}{\mb P(\sigma(i), \sigma(i^\star),A_{i,i^{\star}})}+\log\frac{q_{X}^\star}{p_{ X}(\sigma(i^\star))}\right)^2\nonumber\\
    &=\sum_{i,j>i^{\star}}\log\frac{q_{A}(\sigma(i),A_{i,i^{\star}})}{\mb P(\sigma(i), \sigma(i^\star),A_{i,i^{\star}})}\log\frac{q_{A}(\sigma(j),A_{j,i^{\star}})}{\mb P(\sigma(j), \sigma(i^\star),A_{j,i^{\star}})}+\left(\frac{n-i^\star}{n}\log\frac{q_{X}^\star}{p_{ X}(\sigma(i^\star))}\right)^2\nonumber\\
    &\quad+2\left(\frac{n-i^\star}{n}\log\frac{q_{X}^\star}{p_{ X}(\sigma(i^\star)}\right)\sum_{i>i^\star}\log\frac{q_{A}(\sigma(i),A_{i,i^{\star}})}{\mb P(\sigma(i), \sigma(i^\star),A_{i,i^{\star}})}
\end{align}
We thus complete the upper bound of $\Gamma_1$ by analyzing $\Gamma_{11}$:
\begin{align}\label{ieq:pt2_proof_gamma11}
    \Gamma_{11}&=\E_{\Psi}\biggl[\left(\sum_{i> i^{\star}}^{n} \log\frac{q_{A}(\sigma(i),A_{i,i^{\star}})}{\mb P(\sigma(i), \sigma(i^\star),A_{i,i^{\star}})}+\log\frac{q_{X}^\star}{p_{ X}(\sigma(i^\star))}\right)^2\Bigg|~\sigma(i^\star)=k_1, i^\star\leq \log^2(n)\biggr]\nonumber\\
    &=(n-i^\star)^2 D^2(\bs \alpha, \mb p, \bs\mu)+\sum_{i> i^\star}\E\left[\log^2\frac{q_{A}(\sigma(i),A_{i,i^{\star}})}{\mb P(\sigma(i), \sigma(i^\star),A_{i,i^{\star}})}\right]-\sum_{i> i^\star}\left(\E\left[\log\frac{q_{A}(\sigma(i),A_{i,i^{\star}})}{\mb P(\sigma(i), \sigma(i^\star),A_{i,i^{\star}})}\right]\right)^2\nonumber\\
    &\leq (n-i^\star)^2 D^2(\bs \alpha, \mb p, \bs\mu)+(n-i^\star)\log(\eta_1)D(\bs \alpha, \mb p, \bs\mu).
\end{align}
Plug \eqref{ieq:pt2_proof_gamma11} back to \eqref{ieq:pt2_proof_gamma1} and \eqref{ieq:pt2_proof_exp_L_sq} sequentially we conclude 
\begin{align}
    \E_{\Psi}[L^2]&\leq \frac{\log^2(\eta_1)}{n^2}+\frac{\eta_2^4+\eta_2^2}{n^4}+\left(\log(n)\log(\eta_1)+\eta_2\right)^2+2\left(\log(n)\log(\eta_1)+\eta_2\right)D(\bs \alpha, \mb p, \mu)\nonumber\\
    &\quad+(n-i^\star)^2 D^2(\bs \alpha, \mb p, \bs\mu)+(n-i^\star)\log(\eta_1)D(\bs \alpha, \mb p, \bs\mu)\nonumber\\
    &\lesssim n^2 D^2(\bs \alpha, \mb p, \bs\mu)+\left(n\log(\eta_1)+\eta_2\right)D(\bs \alpha, \mb p, \bs\mu)
\end{align}
Further we end the proof of upper bounding by plugging above inequalities back to \eqref{ieq:pt2_proof_L_variance}:
\begin{align}\label{ieq:pt2_proof_L_variance2}
    &\E_{\Psi}\left[\left(\log\frac{\mathrm{d}\mathbb Q}{\mathrm{d}\mathbb P}-\E_{\Psi}[\log\frac{\mathrm{d}\mathbb Q}{\mathrm{d}\mathbb P}]\right)^2\right]\nonumber\\
    &\leq \left(n+\log(n)\right)\left(\log(\eta_1)D(\bs \alpha, \mb p, \mu)-D^2(\bs \alpha, \mb p, \mu)\right).
\end{align}

\end{itemize}
Finally, we align \eqref{ieq:proof_prt1_proof_s}, \eqref{ieq:pt2_proof_L_variance2} and \eqref{ieq:L_upper} together to conclude     
\begin{align*}
        \liminf \frac{n D(\bs \alpha, \mb P, \boldsymbol{\mu})}{\log(n/s)}\geq 1.
    \end{align*}
Our proof concludes here.
\section{Proof of  \cref{remark_D}}\label{proof_remark_D}
In this section we provide the proof of \cref{remark_D}:
\begin{proof}
    Recall the definition of $D(\bs \alpha,\mb P,\boldsymbol{\mu})$. For any $k_1\neq k_2\in [K]$ define $D'(k_1, k_2,\bs \alpha,\mb P,\boldsymbol{\mu})$ such that:
    \begin{align}\label{def_D'}
       D'(k_1, k_2,\bs \alpha,\mb P,\boldsymbol{\mu}):= \min_{\substack{
    \mb q_{A}\in \mathbb R^{K\times L}\\
    q_{X}\in \mathcal{F}
    }} D_{A}(k_1, k_2,\mb P, \mb q_{A})+\frac{1}{n}D_{X}(k_1, k_2,\bs \mu, q_{X}).
    \end{align}
    Here, $D_{A}(k_1, k_2,\mb P, \mb q_{A})$ and $D_{X}(k_1, k_2,\bs \mu, q_{X})$ follow the definition in \eqref{def:Dx_Da}.
    Thus the proof reduces to showing: \begin{align}
        &D'(k_1, k_2,\bs \alpha,\mb P,\boldsymbol{\mu})\nonumber\\
        &=\max_{t\in [0,1]}\sum_{k}^K\alpha_k\biggl[\sum_{l=1}^{L}\left((1-t)\mb P(k, k_1,l)+t \mb P(k, k_2,l)-\mb P^{1-t}(k, k_1,l)\mb P^{t}(k, k_2,l)\right)+\frac{t(1-t)}{2n}\norm{\bs \mu_{k_1}-\bs\mu_{k_2}}2^2\biggr].
    \end{align}
    We begin the proof by studying the constrained version of \eqref{def_D'}: Note that $D'(k_1, k_2,\bs \alpha,\mb P,\boldsymbol{\mu})$ actually reaches the minimum of the following problem:
    \begin{align}\label{prob:convex_d'}
        &\min_{\substack{
    \mb q_{A}\in \mathbb R^{K\times L}\\
    q_{X}\in \mathcal{F}
    }} \sum_{k=1}^K\alpha_k \KL_{A}\left(\mb q_{A}(k), \mb P(k, k_1)\right)+\frac{1}{n}\KL_{X}\left(q_{X}, p_{ X}(k_1)\right)\nonumber\\
    &\textit{s.t.}\quad \sum_{k=1}^K\alpha_k \KL_{A}\left(\mb q_{A}(k), \mb P(k, k_1)\right)+ \frac{1}{n}\KL_{X}\left(q_{X}, p_{ X}(k_1)\right)\nonumber\\
    &\quad>\sum_{k=1}^K\alpha_k \KL_{A}\left(\mb q_{A}(k), \mb P(k, k_2)\right)+ \frac{1}{n}\KL_{X}\left(q_{X}, p_{ X}(k_2)\right).
    \end{align}
    Thus, we express the lagrange dual form of \cref{prob:convex_d'} such that:
    \begin{align}
        \max_{t\geq 0}\min_{\substack{
    \mb q_{A}\in \mathbb R^{K\times L}\\
    q_{X}\in \mathcal{F}
    }}L(\mb q_{A}, q_{X}, t)
    \end{align}
    where \begin{align}\label{L_def:gamma12}
      &L(\mb q_{A}, q_{X}, t)\nonumber\\
      &=\underbrace{\sum_{k=1}^K\alpha_k \KL_{A}\left(\mb q_{A}(k), \mb P(k, k_1)\right)+t\left(\sum_{k=1}^K\alpha_k \KL_{A}\left(\mb q_{A}(k), \mb P(k, k_2)\right)-\sum_{k=1}^K\alpha_k \KL_{A}\left(\mb q_{A}(k), \mb P(k, k_1)\right)\right)}_{\Gamma_1(\mb q_{\mb A},t)}\nonumber\\
      &\quad+\frac{1}{n}\left[\underbrace{ \KL_{X}\left(q_{X}, p_{ X}(k_1)\right)+t\left( \KL_{X}\left(q_{X}, p_{ X}(k_2)\right)- \KL_{X}\left(q_{X}, p_{ X}(k_1)\right)\right)}_{\Gamma_2(q_{X},t)}\right]
    \end{align}
    for some postive variable $t>0$. Notice that $\Gamma_1\left(\mb q_{ A},t\right)$ has been completely studied in Claim $4$ in \cite{yun2016optimal}, due to the independence of the $\mb q_{A}, q_{X}$, the remaining part then focuses on solving $\Gamma_2(q_{X},t)$.
    
    Since $q_{X}$ represents a PDF that lies in the function space $\mathcal{F}: \mathbb R^{d}\rightarrow \mathbb R$. We denote $\frac{\delta F(p)}{\delta p}(\mb x)$ as the functional derivative of $F(p)$ with respect to $p$, $\forall p\in \mathcal{F}$. By definition, we have
    \begin{align}\label{eqn:condition_qx}
        &\frac{\delta \Gamma_{2}(q_{X}, t)}{\delta q_{X}}(\mb x)\nonumber\\
        &=(1-t)\left[\log(q_{X}(\mb x))+1-\log(p_{ X}(k_1)(\mb x))\right]+t\left[\log(q_{X}(\mb x))+1-\log(p_{ X}(k_2)(\mb x)\right].
    \end{align}
    Here we use $\frac{\delta \KL(q,p)}{\delta q}(\mb x)=\log(q(\mb x))+1-\log(p(\mb x))$ in the derivation. Let $q_X^\star(x):=\argmin_{q_X \in \mathcal{F}}\Gamma_2(q_X, t)$. Setting \eqref{eqn:condition_qx} equal to zero, we have:
    \begin{align}
        \log(q_{X}^\star(\mb x))=(1-t)\log(p_{ X}(k_1)(\mb x))+t\log(p_{ X}(k_2)(\mb x))-1.\nonumber
    \end{align}
    due to the convexity of $\Gamma_2(q_X, t)$.
    Note that $p_{ X}(k_1)$ and $p_{ X}(k_2)$ represent the PDF for $\mathcal{N}(\boldsymbol{\mu}_{k_1}, \mb I_d)$ and $\mathcal{N}(\boldsymbol{\mu}_{k_2}, \mb I_d)$ respectively. Consequently, after some calculation, we obtain 
    \begin{align}\label{eqn:q_X}
        q_{X}^\star(\mb x)=C\exp(-\norm{\mb x-(1-t)\bs \mu_{k_1}-t\bs \mu_{k_2}}2^2/2).
    \end{align}
    after some calculation. Here $C$ is a constant that is irelevent with $\mb x$. We assert $C=(2\pi)^{-d/2}$ by utilizing $q_{X}$ is an probability density function. 
    
    Substituting the expression of $q_{X}^\star$ into $\Gamma_2(q_{X},t)$ defined in \eqref{L_def:gamma12} yields:
    \begin{align}\label{eqn:gamma2}
        \Gamma_2(q_{X}^\star,t)=\frac{t(1-t)}{2}\norm{\bs \mu_{k_1}-\bs\mu_{k_2} }2^2.
    \end{align}
after some algebra. Here, we use      \cref{lem:gaussian_KL} and the fact that $q_{X}^\star$ is the PDF for $\mathcal{N}((1-t)\boldsymbol{\mu}_{k_1}+t\boldsymbol{\mu}_{k_2}, \mb I_d)$ in the derivation. We thus finish the proof by combining the expressions of $\Gamma_1(\mb q_A,t)$(see \cite[Claim $4$]{yun2016optimal}) and  $\Gamma_2(q_{X}^\star,t)$ together. 

\end{proof}
    

\section{Proof of \cref{lem:spec}}\label{sec:proof_lem_spec}
In this section we provide the proof of \cref{lem:spec}:
\begin{proof}
    Let $\mb S^{(K)}$ be the best rank-$K$ approximation of $\mb S$,  define $\hat{\sigma}(\cdot)$ and $\{\mb c_i\}_{i\in [n]}$ such that 
    \begin{align}\label{def:hat_sig}
        \hat{\sigma}(\cdot), \{\hat{\mb c_i}\}_{i\in [K]}:=\argmin_{\substack{\sigma(\cdot)\in \mathcal{F},\\
        \mb c_i\in \mathbb R^n}} \sum_{i\in [K]}\norm{\mb S^{(K)}_i-\mb c_{\sigma(i)}}2^2.
    \end{align} 
    By lemma \ref{lem:rank} we have $\hat{\sigma}=\bar{\sigma}\circ \pi$ for some permutation $\pi$. Thus we could then focues on $\hat{\sigma}$ in the remaining part of the proof. We start by bounding $\norm{\mb S^{(K)}-\mb M}{op}$ for the $\mb M$ defined in \eqref{eqn:exp_s}:
    \begin{align}
        \norm{\mb S^{(K)}-\mb M}{op}&\leq \norm{\mb S^{(K)}-\E\left[\mb S\right]}{op}+\norm{\E\left[\mb S\right]-\mb M}{op}\nonumber\\
        &\leq \norm{\mb S^{(K)}-\mb S}{op}+\norm{\mb S-\E\left[\mb S\right]}{op}+\norm{\E\left[\mb S\right]-\mb M}{op}\nonumber\\
        &\leq 2\norm{\mb S-\E\left[\mb S\right]}{op}+\norm{\E\left[\mb S\right]-\mb M}{op}\nonumber\\
        &\leq 2\left(\norm{\sum_{l=1}^L w^{(l)}\mb A^{(l)}-\E\left[\sum_{l=1}^Lw^{(l)}\mb A^{(l)}\right]}{op}+\norm{\frac{1}{n}\mb X\mb X^T-\frac{1}{n}\E(\mb X\mb X^T)}{op}\right)+1\nonumber\\
        &\lesssim \sqrt{n}.
    \end{align}
    with probability at least $1-cn^{-c'}$. Here we use tha fact that $\mb S^{(K)}$ is the best approximation of $\mb S$ in the third step, \cite[Lemma 8.5]{coja2010graph} and \cite[Lemma B.1]{loffler2021optimality} in the last step.
    Notice that $\rank~(\mb S^{(K)}-\mb M)\leq \rank(\mb S^{(K)})+\rank(\mb M)\leq 2K$ we further have 
    \begin{align}\label{ieq:upper_Sk_M}
        \norm{\mb S^{(K)}-\mb M}{F}\leq \sqrt{2k}\norm{\mb S^{(K)}-\mb M}{op}\lesssim \sqrt{Kn}.
    \end{align}
    with high probability. 
    On the other hand  define $\hat{\mb C}$ such that $\hat{\mb C}_i:=\hat{\mb c}_{\hat{\sigma}(i)}$, consequently we have
    $\hat{\mb C}$ is the solve of \eqref{def:hat_sig} :\begin{align}
        \norm{\hat{\mb C}-\mb S^{(K)}}{F}\leq \norm{\mb M-\mb S^{(K)}}{F}.
    \end{align}
    
    Consequently we have 
    \begin{align}
        \norm{\hat{\mb C}-\mb M}{F}\leq \norm{\hat{\mb C}-\mb S^{(K)}}{F}+\norm{\mb S^{(K)}-\mb M}{F}\leq 2\norm{\mb S^{(K)}-\mb M}{F}\lesssim\sqrt{Kn}.
    \end{align}
    with high probability.
    Here we use \eqref{ieq:upper_Sk_M} in the last step. Define set $S_m(\delta)$ such that \begin{align}
        S_m(\delta):=\left\{i\in [n]~|~\norm{\hat{\mb C_i}-\mb M_i}2\geq \frac{\delta}{2} \right\}.
    \end{align}
    where $\delta:=\min_{\mb M_i\neq \mb M_j}\norm{\mb M_i-\mb M_j}2$. It is easy to check any index $i\in [n]\setminus S_m(\delta)$ will not be misclassfied up to permutation. Thus we finish the proof by upper bound the size of $S_m(\delta)$:\begin{align}\label{ieq:s_upper}
        s(\hat{\sigma})\leq\textit{card}~(S_m(\delta))\leq \frac{4\norm{\hat{\mb C}-\mb M}{F}^2}{\delta^2}\lesssim\frac{kn}{\delta^2}.
    \end{align}
    By the definition of $\delta$ we have 
\begin{align}
    \delta&=\min_{\sigma(i)\neq \sigma(j)}\norm{\mb Z\left(\mb P_s+\frac{1}{n}\bs \mu^T\bs\mu\right)\mb z_i^T-\mb Z\left(\mb P_s+\frac{1}{n}\bs \mu^T\bs\mu\right)\mb z_j^T}2\nonumber\\
    &=\min_{\sigma(i)\neq \sigma(j)}\norm{\mb Z\left(\mb P_s+\frac{1}{n}\bs \mu^T\bs\mu\right)_{\sigma(i)}-\mb Z\left(\mb P_s+\frac{1}{n}\bs \mu^T\bs\mu\right)_{\sigma(j)}}2\nonumber\\
    &=\min_{\sigma(i)\neq \sigma(j)}\norm{\mb Z\bs \Sigma^{-1}\bs\Sigma\left(\left(\mb P_s+\frac{1}{n}\bs \mu^T\bs\mu\right)_{\sigma(i)}-\left(\mb P_s+\frac{1}{n}\bs \mu^T\bs\mu\right)_{\sigma(j)}\right)}2\nonumber\\
    &=\min_{\sigma(i)\neq \sigma(j)}\norm{\bs\Sigma\left(\left(\mb P_s+\frac{1}{n}\bs \mu^T\bs\mu\right)_{\sigma(i)}-\left(\mb P_s+\frac{1}{n}\bs \mu^T\bs\mu\right)_{\sigma(j)}\right)}2
\end{align}
where $\bs\Sigma:=\diag([n_1,n_2,\cdots,n_K])=n\left(\diag(\bs\alpha)+\mb \Delta)\mb I_K\right)$ for some $\mb \Delta$ i.e. $\norm{\mb \Delta}{op}=o(1)$. Here we use the orthonormality for the columns of $\mb Z\bs\Sigma^{-1}$ in the last step. Thus we further obtain\begin{align}\label{ieq:delta}
    \delta=(1+o(1))n\min_{\sigma(i)\neq \sigma(j)}\norm{\left(\mb P_s\right)_{\sigma(i)}-\left(\mb P_s\right)_{\sigma(j)}+\frac{1}{n}\left(\bs \mu_{\sigma(i)}-\bs \mu_{\sigma(j)}\right)^T\bs \mu}2.
\end{align}
Combine \eqref{ieq:s_upper} and \eqref{ieq:delta} together we finish the proof here.
\end{proof}

\section{Technical lemmas}
In this section we provide some technical lemmas that is essential to our main results.
\begin{lemma}\label{lem:D}
    Recall the definition of $D_{A}$, $D_{X}$ and $D(\bs \alpha,\mb P,\boldsymbol{\mu})$, one has 
    \begin{align}\label{eqn:lem_D}
        &D_{A}(k_1^\star, k_2^\star, \mb q_{A}^\star)=\sum_{k=1}^K\alpha_k \KL_{A}\left(\mb q_{A}^\star(k), \mb P(k, k_1)\right)= \sum_{k=1}^K\alpha_k\KL_{A}\left(\mb q_{A}^\star(k), \mb P(k, k_2)\right)\nonumber\\
        \textit{and},\quad&D_{X}(k_1^\star, k_2^\star, q_{X}^\star)=\KL_{X}\left(q_{X}^\star, p_{ X}(k_1)\right)= \KL_{X}\left(q_{X}^\star, p_{ X}(k_2)\right)
    \end{align}
    Here $\KL(\cdot,~\cdot)$ represent the KL divergence and \begin{align}
        (k_1^\star, k_2^\star, \mb q_{A}^\star, q_{X}^\star):=\argmin\limits_{\substack{k_1\neq k_2\\
    \mb q_{A}\in \mathbb R^{K\times L}\\
    q_{X}\in \mathcal{F}
    }} D_{A}(k_1, k_2, \mb q_{A})+\frac{1}{n}D_{X}(k_1, k_2, q_{X})
    \end{align}
\end{lemma}
\begin{proof}
\item \paragraph{Proof of first equation in \eqref{eqn:lem_D}:}
    Here we mimic the approach in \cite{yun2016optimal} and finish the proof by contradiction: 
    Recall the definition of $D_{A}(k_1^\star, k_2^\star, \mb q_{A}^\star)$ we have
    \begin{align}
      D_{A}(k_1^\star, k_2^\star, \mb q_{A}^\star)=\max\left\{\sum_{k=1}^K\alpha_k \KL_{A}\left(\mb q_{A}^\star(k), \mb P(k, k_1)\right), \sum_{k=1}^K\alpha_k\KL_{A}\left(\mb q_{A}^\star(k), \mb P(k, k_2)\right)\right\}.\nonumber
    \end{align}
    Without loss of generality, we assume 
    \begin{align}
D_{A}(k_1^\star, k_2^\star, \mb q_{A}^\star)=\sum_{k=1}^K\bs \alpha _k\KL_{A}\left(\mb q_{A}^\star(k), \mb P(k, k_1^\star)\right)> \sum_{k=1}^K\bs \alpha_k\KL_{A}\left(\mb q_{A}^\star(k), \mb P(k, k_2^\star)\right).\nonumber
    \end{align}
    Thus there should exist a $k'\in [K]$ such that
    \begin{align}
        \KL_{A}\left(\mb q_{A}^\star(k'), \mb P(k', k_1^\star)\right)>\KL_{A}\left(\mb q_{A}^\star(k'), \mb P(k', k_2^\star)\right).\nonumber
    \end{align}
    Leverage the positivity of the KL-divergence we assert $\mb q_{A}^\star(k')\neq \mb P(k', k_1^\star)$ then there should exist a $\mb q_{A}'(\cdot,~\cdot)\in \mathbb R^{K\times L}$ satisfy the following two conditions
    \begin{align*}
        &\begin{cases}
            \KL_{A}(\mb q_{A}'(k),\mb P(k, k_1^\star))=\KL_{A}(\mb q_{A}^\star(k),\mb P(k, k_1^\star))&k\neq k'\\
            \KL_{A}(\mb q_{A}'(k),\mb P(k, k_1^\star))\in [\KL_{A}(\mb q_{A}^\star(k'),\mb P(k', k_1^\star))-\eps,\KL_{A}(\mb q_{A}^\star(k'),\mb P(k', k_1^\star))]&k=k'
        \end{cases}\\
        &\textit{and}\\
        &\begin{cases}
            \KL_{A}(\mb q_{A}'(k),\mb P(k, k_2^\star))=\KL_{A}(\mb q_{A}^\star(k),\mb P(k, k_2^\star))&k\neq k'\\
            \KL_{A}(\mb q_{A}'(k),\mb P(k, k_2^\star))\in [\KL_{A}(\mb q_{A}^\star(k'),\mb P(k', k_2^\star)),\KL_{A}(\mb q_{A}^\star(k'),\mb P(k', k_2^\star))+\eps]&k=k'
        \end{cases}
    \end{align*}
    where $\eps$ is an constant such that $\eps\in[0, 1/2\left(\KL_{A}\left(\mb q_{A}^\star(k'), \mb P(k', k_1^\star)\right)-\KL_{A}\left(\mb q_{A}^\star(k'), \mb P(k', k_2^\star)\right)\right)]$. Then we have 
    \begin{align}
        D_{A}(k_1^\star, k_2^\star, \mb q_{A}^\star)>\underbrace{\sum_{k=1}^K \bs \alpha _k\KL_{A}\left(\mb q_{A}'(k), \mb P(k, k_1^\star)\right)}_{:=D_{A}(k_1^\star, k_2^\star, \mb q_{A}')}> \sum_{k=1}^K\bs \alpha_k\KL_{A}\left(\mb q_{A}'(k), \mb P(k, k_2^\star)\right).\nonumber
    \end{align}
    thus further leads to:
    \begin{align}
        D(\bs \alpha,\mb P,\boldsymbol{\mu})=D_{A}(k_1^\star, k_2^\star, \mb q_{A}^\star)+D_{X}(k_1^\star, k_2^\star, \mb q_X^\star)>D_{A}(k_1^\star, k_2^\star, \mb q_{A}')+D_{X}(k_1^\star, k_2^\star, \mb q_X^\star).\nonumber
    \end{align} 
    which is contradicted by the definition of $D(\bs \alpha,\mb P,\boldsymbol{\mu})$.
\item \paragraph{Proof of second equation in \eqref{eqn:lem_D}:} The proof here is in spirits similar to the previous part via contradicting the definition of $D(\bs \alpha,\mb P,\boldsymbol{\mu})$ by constructing a new distribution function $q_{X}'$:
Recall the definition of the $D_{X}(k_1^\star, k_2^\star, \mb q_{A}^\star)$:
\begin{align}
    D_{X}(k_1^\star, k_2^\star, q_{X}^\star):=\max\left\{ \KL_{X}\left(q_{X}^\star, p_{ X}(k_1^\star)\right), \KL_{X}\left(q_{X}^\star, p_{ X}(k_2^\star)\right)\right\}.\nonumber
\end{align}
W.L.O.G we assume 
\begin{align}
 \KL_{X}\left(q_{X}^\star(k), p_{ X}(k_1^\star)\right)>\KL_{X}\left(q_{X}^\star(k), p_{ X}(k_2^\star)\right).
\end{align}
Similarly to the first part, we could also build a new measure $q_{X}'(\cdot,\cdot)$ such that 
\begin{align}
    D_{X}(k_1^\star, k_2^\star, q_{X}^\star)>\underbrace{\KL_{X}\left(q_{X}', p_{ X}(k_1^\star)\right)}_{:=D_{X}(k_1^\star, k_2^\star,q_{X}')}>\KL_{X}\left(q_{X}'(k), p_{ X}(k_2^\star)\right).\nonumber
\end{align}
thus further leads to:
    \begin{align}
        D(\bs \alpha,\mb P,\boldsymbol{\mu})=D_{A}(k_1^\star, k_2^\star, \mb q_{A}^\star)+D_{X}(k_1^\star, k_2^\star, \mb q_X^\star)>D_{A}(k_1^\star, k_2^\star, \mb q_{A}^\star)+D_{X}(k_1^\star, k_2^\star, \mb q_X').\nonumber
    \end{align} 
    which is a contradiction against the definition of $D(\bs \alpha,\mb P,\boldsymbol{\mu})$. Until then our proof ends.
\end{proof}
\begin{lemma}[KL divergence for two Gaussian distributions]\label{lem:gaussian_KL}
    Consider two Gaussian probability distributions with unit variance, having means $\bs \mu_1$ and $\bs \mu_2$, and PDFs denoted as $p_1$ and $p_2$, respectively.  It follows that one has 
    \begin{align}
        \KL(p_1,p_2)=\frac{\norm{\bs\mu_1-\bs\mu_2}2^2}{2}.
    \end{align}
    Here $\KL(p,q):=\E_{p}\log(\frac{p}{q})$ represents the KL-divergence for two continuous probability distribution.
\end{lemma}
\begin{proof}
    By definition we have
    \begin{align}
        \KL(p_1,p_2)&=\E_{\mb x\sim p_1}\log(\frac{p_1(\mb x)}{p_2(\mb x)})\nonumber\\
        &=\frac{1}{2}\E_{\mb x\sim p_1}\left(\norm{\mb x-\bs \mu_2}2^2-\norm{\mb x-\bs \mu_1}2^2\right)\nonumber\\
        &=\frac{1}{2}\left(\norm{\bs \mu_2}2^2-\norm{\bs \mu_1}2^2\right)+\E_{x\sim p_1}\left(\bs \mu_1-\bs \mu_2\right)^T\mb x\nonumber\\
        &=\frac{1}{2}\left(\norm{\bs \mu_2}2^2-\norm{\bs \mu_1}2^2\right)+\left(\bs \mu_1-\bs \mu_2\right)^T\bs \mu_1\nonumber\\
        &=\frac{\norm{\bs\mu_1-\bs\mu_2}2^2}{2}.
    \end{align}
\end{proof}
    
\begin{lemma}\label{lem:rank}
    Let $\mb S^{(K)}$ as the best rank-$k$ approximation of $\mb S$  proposed in algorithm \ref{alg:sp}, denote $\hat{\sigma}(\cdot)$ and $\{\mb c_i\}_{i\in [n]}$ such that:
    \begin{align}
        \hat{\sigma}(\cdot), \{\mb c_i\}_{i\in [n]}:=\argmin_{\substack{\hat{\sigma}(\cdot)\in \mathcal{F},\\
        \mb c_i\in \mathbb R^n}} \sum_{i\in [n]}\norm{\mb S^{(K)}_i-\mb c_{\hat{\sigma}(i)}}2^2.
    \end{align}
    Recall the definition of $\bar{\sigma}(\cdot)$ in \eqref{alg:sp} we have $\pi(\bar{\sigma}(i))=\hat{\sigma}(i)$ for some permutation $\pi(\cdot)$.
\end{lemma}
\begin{proof}
    By definition of $\bar{\sigma}(\cdot)$ and $\mb Q$ in \cref{alg:sp} we have 
    \begin{align*}
        \sum_{i\in [n]}\norm{\mb Q_i-\bs \theta_{\bar{\sigma}(i)}}2^2&=\sum_{i\in [n]}\norm{\mb U_K\mb Q^{(K)}_i-\mb U_k \bs \theta_{\bar{\sigma}(i)}}2^2\nonumber\\
        &=\sum_{i\in [n]}\norm{\mb S^{(K)}_i-\mb U_K \bs \theta_{\bar{\sigma}(i)}}2^2
    \end{align*}
    thus there exist $\hat{\sigma}(i)$ such that $\mb c_{\hat{\sigma}(i)}:=\mb U_K \bs \theta_{\bar{\sigma}(i)}$. Here we use $\mb U_K$ is a orthogional matrix is the first step and definition of $\mb S^{(K)}$ in the second step.
\end{proof}
\end{document}